\documentclass[times,review,10pt]{elsarticle}

\usepackage{amssymb}
\usepackage{amsmath}

\usepackage{siunitx}
\usepackage{xspace}
\usepackage{lipsum}

\usepackage{adjustbox}
\usepackage{amsthm}
\usepackage{subcaption}
\usepackage[numbers]{natbib}
\usepackage{amsmath,amsfonts}
\usepackage[ruled,vlined]{algorithm2e}
\usepackage{caption}
\usepackage{varwidth}
\usepackage{array}
\usepackage{amsmath}
\usepackage{amssymb}
\usepackage{mathtools}
\usepackage{booktabs}
\usepackage{multirow}

\usepackage{amsmath}
\usepackage{algorithmic}

\usepackage{url}

\newtheorem{theorem}{Theorem}
 
\newtheorem{definition}{Definition}
\newtheorem{intuition}{Intuition}

\newcommand{\N}{\mathcal{N}}
\newcommand{\B}{\mathcal{B}}

\def\tsc#1{\csdef{#1}{\textsc{\lowercase{#1}}\xspace}}
\tsc{WGM}
\tsc{QE}
\tsc{EP}
\tsc{PMS}
\tsc{BEC}
\tsc{DE}
\journal{Knowledge-Based Systems}

\begin{document}

\begin{frontmatter}

\title{Collaborative Weighting with Pessimistic Critic for Mitigating Overestimation in Off-Policy Reinforcement Learning}


\author[1]{Gong Gao}
\ead{2310925@tongji.edu.cn}


\affiliation[1]{organization={School of Computer Science and Technology, Tongji University},
                postcode={200092},
              city={Shanghai},
                country={China}}

  
     \author[2]{Xiao Lai*} 
   \ead{laixiao@sues.edu.cn}
\author[1]{Ziqi Xie} 

\author[1]{Guojie Chen} 

\author[1]{Xianhui Liu} 
\author[1]{Weidong Zhao} 

\affiliation[2]{organization={Shanghai University of Engineering Science},
                postcode={201620‌}, 
                city={Shanghai},
                country={China}}

\cortext[1]{Xiao Lai}

 \begin{abstract}
Deep off-policy reinforcement learning algorithms for continuous control typically rely on neural value function approximation to guide policy improvement. However, temporal-difference (TD) learning introduces noisy targets, resulting in non-stationary optimization, while greedy policy updates amplify early-stage estimation errors. The recursive propagation of such errors leads to persistent overestimation bias and degraded training stability in actor-critic methods. Existing approaches attempt to alleviate this issue via prioritized sampling or modified value learning objectives, but often overemphasize high-uncertainty transitions caused by limited data coverage or bootstrapping errors, thereby further amplifying bias.
In this paper, we propose Collaborative Weighting Actor-Critic (CWAC), a unified framework that explicitly accounts for predictive uncertainty in value estimation. CWAC employs distributional critic to model return uncertainty and introduces a collaborative weighting mechanism that jointly reweights TD-errors and uncertainty, enabling robust learning from reliable samples while suppressing noisy updates. In addition, we incorporate a stochastic pessimistic value estimation scheme via sampling from the return distribution, which effectively mitigates error propagation during policy improvement.
CWAC can be seamlessly integrated into existing off-policy algorithm frameworks such as SAC, TD3, and DDPG with minimal overhead. Empirical results demonstrate that our proposed method significantly enhances performance across a diverse range of simulated tasks.
\end{abstract}


\end{frontmatter}
\section{Introduction}
\label{Introduction}
Off-policy reinforcement learning (RL) algorithms achieve high sample efficiency and have been widely adopted in robot control~\cite{xiao2024reinforcement,kim2026robot}, autonomous driving~\cite{shi2023physics,yang2026raw2drive}, and industrial scheduling~\cite{liu2025knowledge,ngwu2026reinforcement,hua2025solving}. These methods estimate long-term discounted returns using $Q$-functions trained from reused experience in a replay buffer. Modern actor-critic (AC) frameworks~\citep{haarnoja2018soft,fujimoto2018addressing} parameterize these $Q$-functions with deep neural networks and optimize them via temporal-difference (TD) learning~\citep{tesauro1995temporal}. 


A key challenge in off-policy RL is the weak learning signal arising from approximate temporal-difference targets, together with the propagation and amplification of estimation errors during bootstrapping. As illustrated in Figure~\ref{fig:intro_td_error}, the greedy operator $\max_a Q^k(s, a)$ inherently favors actions associated with positive estimation noise.
This selection bias amplifies early estimation errors, triggering a self-reinforcing feedback loop that leads to systematic overestimation and undermines training stability.
\begin{figure*}[ht]
\begin{center}
\centerline{\includegraphics[width=1.0\textwidth]{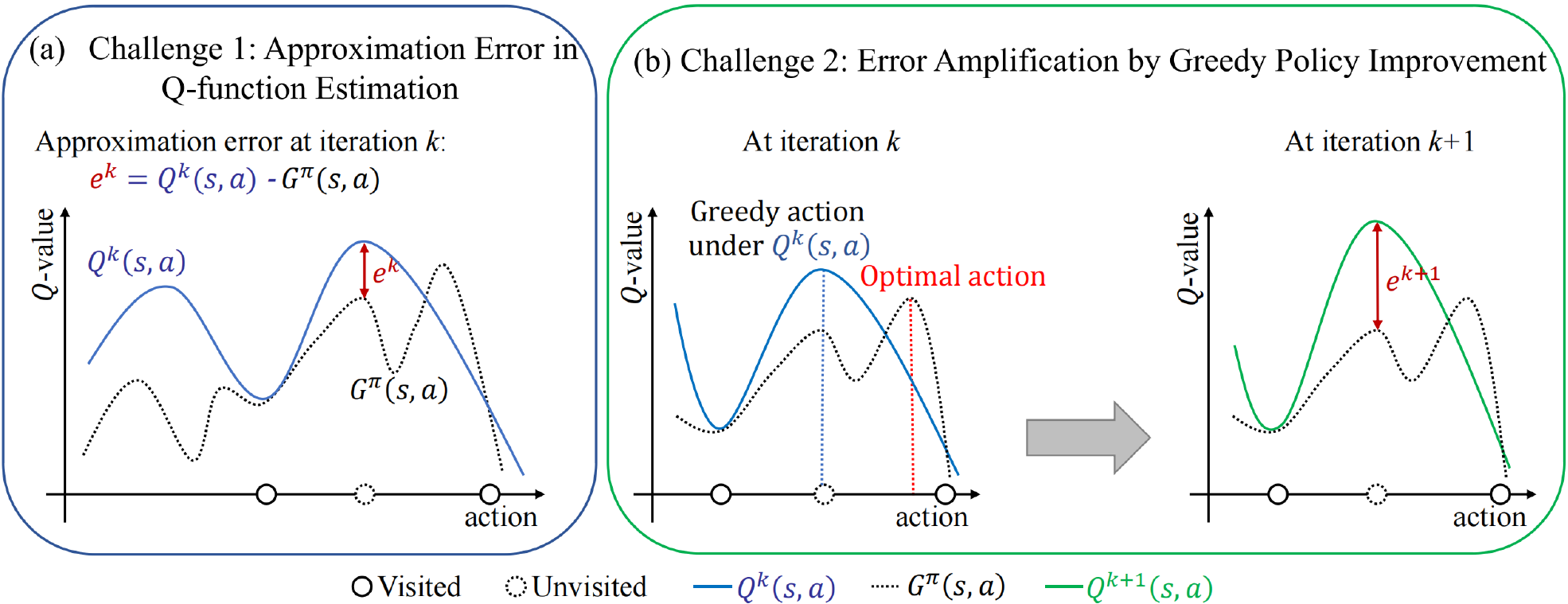}}
\caption{Illustration of two fundamental challenges in off-policy RL: (a) approximation errors in Q-function estimation, and (b) error amplification induced by greedy policy improvement. Here, $Q^k$ and $Q^{k+1}$ denote the value estimates at the $k$-th and $(k+1)$-th training iterations, respectively. $G^\pi(s,a)$ represents the ground-truth action-value function under policy $\pi$, defined by the standard infinite-horizon discounted return as
$
G^\pi(s,a)
=
\mathbb{E}_{a\sim \pi_\phi}
\left[
\sum_{t=0}^{\infty}
\gamma^t \mathcal{R}(s,a)
\;\middle|\;
s_0=s,\, a_0=a
\right].
$
}
\label{fig:intro_td_error}
\end{center}
\vskip -0.2in
\end{figure*}

Existing methods typically tackle this issue from two perspectives: pessimistic value estimation and adaptive sample weighting. Pessimistic approaches~\citep{cetin2023learning,nauman2024overestimation,wang2025dynamic} mitigate overestimation by penalizing high-uncertainty actions. While a moderate degree of pessimism can be effective, excessive conservatism in online settings may hinder exploration, leading to suboptimal performance~\citep{karimpanal2023balanced,tan2024adaptive}. On the other hand, adaptive sample weighting methods emphasize transitions with large TD-error~\citep{schaul2015prioritized,fujimoto2020equivalence,zhu2023importance}. However, such signals are often entangled with uncertainty, which can amplify noisy interactions and introduce biased updates during training.

To address these challenges, a lightweight uncertainty modeling approach based on a distributional critic is first introduced to estimate predictive uncertainty from return distributions. Building upon this formulation, Collaborative Weighting Actor-Critic (CWAC) is proposed, integrating stochastic pessimistic value estimation with an uncertainty-aware collaborative weighting mechanism. The proposed collaborative weighting jointly and adaptively calibrates the influence of TD-errors and predictive uncertainty: predictive uncertainty reweights TD-errors to suppress unreliable learning signals, while TD-errors reciprocally refine uncertainty estimation. In addition, adaptive pessimistic value maximization is achieved via stochastic sampling from the return distribution, reducing value overestimation and error propagation while maintaining sufficient exploration. As a result, CWAC improves both sample efficiency and training stability in off-policy reinforcement learning.
Our contributions can be summarized in threefold:
\begin{itemize}
\item We propose CWAC, which incorporates a collaborative weighting mechanism that adaptively favors optimization according to estimation error and predictive uncertainty, thereby suppressing unreliable updates and improving learning stability.
\item We develop a stochastic pessimistic value estimation scheme based on distributional critic, where policy improvement is performed via sampling from the return distribution rather than direct maximization of expected Q-values, thereby mitigating error amplification induced by greedy updates.
\item Extensive experiments on continuous control benchmarks demonstrate that CWAC achieves competitive or superior performance compared to state-of-the-art methods, while significantly improving sample efficiency and training stability.
\end{itemize}

The remainder of this paper is organized as follows. Section~\ref{sec:pre} introduces the background. Related work is reviewed in Section~\ref{related_wor}. In Section~\ref{sec:method}, we present the proposed Collaborative Weighting Actor-Criti (CWAC), including the overall framework and implementation details. Section~\ref{sec:experiment} reports experimental results on 8 challenging continuous control tasks, demonstrating the effectiveness and practicality of CWAC. Finally, Section~\ref{conclusion} concludes the paper and outlines potential future directions.

\section{Preliminaries}
\label{sec:pre}
\textbf{Reinforcement Learning.}
Within the standard framework of the Markov decision process (MDP), RL can be formulated as $\mathcal{M}=\langle\mathcal{S}, \mathcal{A}, \mathcal{P}, \mathcal{R}, \gamma\rangle$. Here, ${\cal S}$ denotes the state space, ${\cal A}$ denotes the action space, $\mathcal{P}(\cdot \mid s,a)$ stands for transition dynamics, $\mathcal{R}: {\cal S\times A}\to \mathbb{R}$ denotes the reward function, and $\gamma\in (0,1]$ is the discount factor. Reinforcement learning aims at finding a policy $\pi(\cdot \mid s)$ such that the expected cumulative long-term rewards $J_\pi(\phi) = \mathbb{E}_{s \sim \mathcal{B}, a \sim \pi_\phi}[\sum_{t=0}^\infty\gamma^t \mathcal{R}(s_t,a_t)]$ are maximized. The state and state-action distributions induced by $\pi$ are denoted by $\rho_{\pi}(s)$ and $\rho_{\pi}(s, a)$, respectively.

The optimal policy can be obtained via a maximum entropy variant of policy iteration, which comprises two alternating steps: (a) policy evaluation and (b) policy improvement. Given a policy $\pi$, the corresponding Q-function is learned via the soft Bellman operator $\mathcal{T}^\pi$:
\begin{equation}
\label{eq.soft_bellman}
\mathcal{T}^\pi Q(s,a)=\mathcal{R}(s,a)+\gamma \mathbb{E}_{s^{'}\sim \rho_\pi,a^{'}\sim \pi_{\phi}}\left[Q_{\theta ^{'} }(s^{'},a^{'}) 
-\alpha \log\pi_{\phi} (a^{'}|s^{'})\right],
\end{equation}
where $\theta$ denotes the parameters of the Q-network, and $\phi$ denotes the parameters of the Gaussian policy. The entropy coefficient $\alpha$ mitigates overestimation bias by
penalizing overly deterministic policies.
The parameters of the Q-function are optimized by minimizing the soft Bellman residual, which can be formalized as
\begin{equation}
\mathcal{L}_{Q}(\theta)
=
\mathbb{E}_{(s,a)\sim\mathcal{B}}
\left[
\left(
Q_\theta(s,a) - \mathcal{T}^\pi Q(s,a)
\right)^2
\right].
\end{equation}

\textbf{Weighted Q-learning.}
Weighted Q-learning~\cite{zhang2017weighted} introduces adaptive importance weights into temporal-difference learning to mitigate value estimation bias. Instead of treating all transitions equally, weighted Q-learning assigns state-action dependent coefficients to modulate the contribution of each sample during value updates. This framework integrates a range of approaches, including prioritized experience replay~\citep{schaul2015prioritized} and uncertainty-aware weighted learning methods~\citep{wu2021uncertainty}.

Weighted Q-learning generalizes this objective by introducing an adaptive weight function $\omega(s, a)$
\begin{equation}
\label{eq:weighting_q_learning}
\mathcal{L}_{Q}^{\omega}(\theta)
=
\mathbb{E}_{(s,a)\sim\mathcal{B}}
\left[
\omega(s,a) \cdot
\left(
Q_\theta(s,a) - \mathcal{T}^\pi Q(s,a)
\right)^2
\right],
\end{equation}
where $\omega(s,a) $ modulates the contribution of each transition to the
critic update.

\section{Related Work}
\label{related_wor}
\textbf{Overestimation Bias in Reinforcement Learning.}
Overestimation bias originates from the interaction between stochastic value estimation errors and the maximization operator in TD learning. Even in tabular settings, applying the $\max$ operator to noisy estimates induces systematic overestimation. While DQN~\citep{mnih2015human} improves training stability via target networks, it does not eliminate this issue, which is further exacerbated under function approximation. To address this, Double Q-learning~\citep{van2016deep} decouples action selection from evaluation to reduce maximization bias. Building on this idea, Maxmin DQN~\citep{lan2020maxmin} explicitly mitigates overestimation by maintaining an ensemble of Q-networks and using the minimum overestimates, providing a more conservative value approximation.

In continuous control, actor–critic methods further mitigate overestimation via coordinated architectural and algorithmic refinements. For instance, TD3~\citep{fujimoto2018addressing} incorporates clipped double Q-learning, delayed policy updates, and target policy smoothing to reduce estimation bias. Similarly, maximum entropy methods such as SAC~\citep{haarnoja2018soft} introduce entropy regularization to soften policy evaluation, thereby improving robustness while partially alleviating overestimation.
Nevertheless, these approaches still suffer from overestimation bias. For instance, \citep{quangaugmenting} and \citep{gao2026improving} argue that off-policy algorithms exhibit value overestimation, which hampers efficient policy exploitation, and accordingly propose to improve the decision-making process to enhance sample efficiency.
In contrast,~\citep{oren2025value} attributes suboptimal performance of neural network–parameterized actor–critic methods to insufficiently strong value guidance, and advocates for more aggressive, greedier value evaluation to provide stronger policy improvement signals.
However, these methods fail to tackle the root cause of overestimation: the coupling between function approximation errors in parameterized Q-networks and the error amplification introduced by greedy policy updates.

\textbf{Uncertainty-Aware Value Estimation and Pessimism.}
 A prominent line of work addresses overestimation bias via uncertainty-aware value estimation. These approaches introduce pessimism by penalizing high-uncertainty Q-values or by optimizing lower confidence bounds. For instance, \citep{wu2021uncertainty,bai2022pessimistic} adopt Bayesian formulations to quantify epistemic uncertainty and construct pessimistic value estimates, thereby mitigating overestimation bias. \citep{gong2023adaptive} further leverages both uncertainty and familiarity to regulate value estimation, enabling adaptive adjustment across state–action pairs and effectively suppressing error amplification.

Building on this direction, \citep{guo2022model,wang2025dynamic} propose to dynamically adapt the degree of pessimism, treating it as either a learnable parameter or a constrained variable to better balance conservatism and exploration. Moreover, \citep{moskovitz2021tactical} provide empirical evidence supporting this perspective, while \citep{karimpanal2023balanced,tan2024adaptive} establish theoretically that estimation bias is not universally detrimental; in certain environments, a controlled degree of overestimation can even be beneficial.
In contrast to prior approaches, the proposed CWAC algorithm introduces adaptive stochastic pessimism into value estimation, effectively mitigating overestimation while preserving exploratory capacity.



\textbf{Importance Weighting Mechanisms.}
A line of work seeks to improve sample efficiency by reweighting transitions to emphasize informative updates. Existing approaches often upweight samples using advantage-guided schemes~\citep{peng2019advantage,nair2020awac,chen2022lapo} or energy-based criteria~\citep{zhang2025energyweighted,alles2025flowq}. However, these approaches typically lack explicit quantification of the uncertainty associated with the guiding signals, which may result in misleading updates when the guidance is inaccurate.

Another line of research focus on adaptive reweighting strategies that prioritize transitions deemed critical for learning. Representative approaches include prioritized experience replay (PER)~\citep{schaul2015prioritized,fujimoto2020equivalence,hassani2025improved}. PER~\citep{schaul2015prioritized} prioritizes samples according to the magnitude of the TD-error, which, however, may overemphasize noisy transitions. LAP~\citep{fujimoto2020equivalence} refines this strategy by adopting loss-aware prioritization, selecting samples based on gradient magnitude.
BETDQNet~\citep{hassani2025improved} further constructs more reliable weighting schemes by jointly considering one-step TD-errors and state-wise averaged errors.
While these methods can alleviate Q-value overestimation to some extent, TD-error is often correlated with estimation uncertainty, and thus, such priority-based mechanisms may inadvertently amplify bias and induce training instability if not properly controlled.


\section{Methods}
\label{sec:method}

\subsection{Stochastic Pessimistic Value Estimation}
\begin{definition}[Stochastic Pessimistic Value Sampling]
\label{thm:pessimism_exploration}
Given distributional critic to characterize the return distribution via its mean $Q_\theta(s, a)$ and standard deviation $\sigma_\theta(s, a)$.
Let the biased Q-value estimator be defined as:
\begin{equation}
\label{eq:pessimistic_q}
    \mathcal{Z}_{\theta}(s,a) = Q_{\theta}(s,a) - b(\sigma_\theta(s,a)),
\end{equation}  
where $b(\sigma)=|\epsilon|\cdot \sigma$ denotes a stochastically sampled value, $\epsilon\sim \mathcal{N}(0,\mu \cdot \mathrm{I})$ is a Gaussian noise term sampled from a zero-mean Gaussian distribution with covariance $\mu \cdot \mathrm{I}$, where $\mu$ controlling the magnitude of stochastic perturbations.

\end{definition}

\begin{intuition}
Consider the decomposition of the action-value estimate $Q_{\theta}(s, a) = G^\pi(s, a) + e(s, a)$, where $e(\cdot)$ denotes the function approximation error. 
By substituting the raw Q-value $Q_{\theta}$ with the pessimistic estimator $\mathcal{Z}_{\theta} = Q_{\theta}(s, a) - b(\sigma_\theta) = G^\pi(s, a) + e(s, a) - b(\sigma_\theta)$, the residual between the predicted and ground-truth action-values under policy $\pi$ is redefined as $\tilde{e}(s, a) = e(s, a) - b(\sigma_\theta)$. The resulting learning dynamics are thus governed by two competing mechanisms:

 \textbf{Mitigation of Error Amplification (Appropriate Bias).} A properly moderate bias shifts the mean of the error distribution $\tilde{e}$ such that $\mathbb{E}[\max \mathcal{Z}_{\theta}] \approx G^\pi$. This neutralizing effect arrests the recursive amplification of approximation errors, effectively mitigating the overestimation bias.

 \textbf{Suppression of Exploration (Excessive Pessimism).}
Let $\mathcal{D}_{\text{train}}$ denote the replay buffer and $\mathcal{D}_{\text{test}}$ denote the rollout distribution. 
We assume the uncertainty of value estimation satisfies
$
\sigma(s,a)_{(s,a)\sim \mathcal{D}_{\text{train}}}
\;\le\;
\sigma(s,a)_{(s,a) \sim \mathcal{D}_{\text{test}}},
$
reflecting higher uncertainty in unvisited interaction.
Consider a state $s$ and two actions $a_1$ and $a_2$. The state-action pair $(s, a_1)$ corresponds to a previously observed interaction, i.e., $(s, a_1) \sim \mathcal{D}_{\text{train}}$, whereas $(s, a_2)$ corresponds to an unvisited interaction, i.e., $(s, a_2) \sim \mathcal{D}_{\text{test}}$.
Assume both actions have comparable true value:
$
G^\pi(s,a_1)<
G^\pi(s,a_2)
,
$
so that $a_2$ is the truly better action. 

Under pessimistic estimation, we introduce a pessimistically adjusted value $\mathcal{Z}$ by penalizing the Q-value with uncertainty. The difference between the pessimistic values is given by
$
\mathcal{Z}(s,a_2)-\mathcal{Z}(s,a_1)
=  
\tilde{e}(s,a_2)
-
\tilde{e}(s,a_1)  .
$
Consider a simple deterministic MDP with initial state $s$ and two actions $a_1$ and $a_2$. The true returns satisfy $G^\pi(s,a_1)=1$ and $G^\pi(s,a_2)=2$. Suppose the approximation errors are $e(s,a_1)=0.01$ and $e(s,a_2)=0.1$, with predictive uncertainties $b(\sigma(s,a_1))=0.01$ and $b(\sigma(s,a_2))=0.5$. 
Under pessimistic adjustment, the effective errors become $\tilde{e}(s,a_1)=0$ and $\tilde{e}(s,a_2)=-0.4$. Consequently,
$
\mathcal{Z}(s,a_1) > \mathcal{Z}(s,a_2),
$
indicating that excessive pessimism can reverse the true action preference and suppress the exploratory action.

 
\end{intuition}

\subsection{Collaborative Weighting Actor-Critic}

\begin{theorem}[Weighted Q-learning is not enough]
\label{thm:weighting_failure}
Consider the weighted TD-error objective $\mathcal{L}^{\omega} = \mathbb{E}[\omega(s,a) \cdot \delta(s,a)^2]$, where the weighting function $\omega(s,a) \propto |\delta(s,a)|^p$ (e.g., Prioritized Experience Replay~\cite{schaul2015prioritized}). 
Suppose the soft Bellman target $\mathcal{T}^\pi Q$ is corrupted by heteroscedastic noise $\eta(s, a)$ with $\mathbb{E}[\eta] > 0$, the prioritization weight $\omega(s, a)$ induces a biased gradient update that increases with the noise magnitude. Consequently, high-variance state-action regions are systematically overemphasized, amplifying overestimation bias and potentially leading to unstable or divergent value estimates.
\end{theorem}

\textbf{Proof sketch:}
Consider the observed TD-error $\delta = \delta^\star + \eta$, where $\delta^\star$ is the true residual and $\eta$ is the approximation noise. The weighted gradient is given by $\nabla_\theta \mathcal{L}^\omega \approx \mathbb{E}[|\delta^\star + \eta|^p \cdot (\delta^\star + \eta)  \cdot\nabla_\theta Q_\theta]$. Due to the positive correlation between the weight $|\delta^\star + \eta|^p$ and the noise $\eta$, the expected update is skewed toward samples where $\eta$ is large and positive. The exponent $p$ controls the degree of this weighting effect. In RL, these high-residual samples are often located in unvisited or high-uncertainty regions. Consequently, the optimizer disproportionately updates the critic toward these noisy targets, effectively transforming stochastic noise into systemic overestimation.

\textbf{Why is this a problem?} The fundamental flaw of heuristic weighting based purely on $|\delta|$ is its inability to distinguish between learning progress and stochastic noise. In the nascent stages of training or in complex environments with high epistemic uncertainty, large TD-errors are frequently artifacts of function approximation limits or aleatoric noise rather than meaningful Bellman residuals. By reinforcing these ``hard" samples, the agent triggers a positive feedback loop: overestimation leads to higher prioritization, which in turn leads to further overestimation. This instability prevents the critic from grounding the value surface, ultimately degrading policy optimization.

\begin{definition}[Collaborative Weighting Value Learning]
\label{defi:cwvl}
Consider the gradient of the collaborative loss function $\mathcal{L}_{\mathcal{Z}}^{\omega,\xi}$ with respect to the parameters $\theta$:
\begin{equation}
\label{cwac:eq}
\mathcal{L}_{\mathcal{Z}}^{\omega,\xi}(\theta)
= \mathbb{E}_{(s,a)\sim \mathcal{B},s^{'}\sim \rho_\pi,a^{'}\sim \pi_{\phi}}
\left[
\omega(s,a)\cdot \,
\mathrm{Huber}\left(
Q_\theta(s,a) - y\
\right)+\xi(s,a)\cdot\sigma(s,a)
\right],
\end{equation}
where $y$ denotes the bootstrapped learning target. To improve robustness against noisy value estimates, the critic is optimized using the Huber loss~\citep{huber1992robust} rather than the standard MSE objective. The target value function $Q_{\theta'}$ is further replaced by its stochastic pessimistic counterpart $\mathcal{Z}_{\theta'}$:
\begin{equation}
    y= r + \gamma \left[ \min_{i=1,2}Z_{\theta^{'}_i}(s', a') - \alpha \log \pi_{\phi}(a' \mid s')\right], a' \sim \pi_{\phi}(\cdot \mid s').
\end{equation}

We define the TD-error as $\delta(s,a)=Q_\theta(s,a) - y$, based on which we further construct uncertainty-aware and TD-error-driven weighting functions
\begin{equation}
\label{eq:weights}
\omega(s,a) = \left( \frac{\mathbb{E}[\sigma_{\theta}(s,a)]}{  \sigma_{\theta}(s,a)+c  } \right)^{\beta_\omega},
\quad
\xi(s,a) = \left( \frac{\mathbb{E}[\mid \delta(s,a)\mid]}{\mid \delta(s,a) \mid+c} \right)^{\beta_\xi},
\end{equation}
where $\beta_\omega$ and $\beta_\xi$ control the sharpness of uncertainty and TD-error reweighting, respectively, and $c$ is a small constant for numerical stability.
 
\end{definition}

\begin{intuition}
The corresponding gradient is given by
$\nabla_\theta \mathcal{L}_{\mathcal{Z}}^{\omega,\xi}
= \mathbb{E}_{s \sim \mathcal{B}, a \sim \pi_{\phi}}
[
2\omega \cdot \delta \cdot \nabla_\theta Q_\theta
$\quad $+
\xi \cdot \nabla_\theta \sigma_\theta
]$.
The resulting optimization dynamics are adaptively governed by reciprocal interactions between the weighting terms and the learning targets, establishing a critical equilibrium between overestimation suppression and efficient value learning.

\textbf{Uncertainty-Aware Value Regularization.}
In high-noise or unvisited regions where epistemic uncertainty $\sigma_{\theta} \uparrow$, the weighting coefficient $\omega \downarrow$. This mechanism introduces a principled adaptive pessimism: it suppresses the recursive amplification of overestimation bias by attenuating the influence of unreliable, high-variance samples on the value surface $Q_\theta$.

\textbf{Error-Conditioned Uncertainty Optimization.}
 When the model encounters significant Bellman residuals $|\delta| \uparrow$, the weight $\xi \downarrow$. This serves as a vital safeguard against excessive pessimism: by preventing the premature minimization of $\sigma_\theta$ (i.e., avoiding over-confident uncertainty collapse), the agent maintains a sufficiently diverse uncertainty manifold. This ensures that exploratory gradients are not prematurely extinguished, allowing the agent to continue gathering informative transitions until the value estimate $Q_\theta$ is sufficiently grounded.

\textbf{Gradient Synthesis and Stability.}
As learning converges ($|\delta| \downarrow, \sigma_{\theta}\downarrow$), the weights reciprocally increase ($\xi \uparrow, \omega \uparrow$), facilitating fine-grained refinement of the value surface. This dual-track regulation ensures that the total gradient norm remains bounded and is preferentially oriented toward state-action regions with high statistical confidence, achieving a robust trade-off between conservative value estimation and active exploration.
\end{intuition}

\begin{definition}[Mitigating Maximization Bias in Actor Updates]
\label{def:cwac_actor}
Based on the stochastic pessimistic sampling in Eq.~\eqref{eq:pessimistic_q}, we further incorporate pessimism into policy improvement. The actor's objective is optimized by maximizing
\begin{equation}
\label{eq:actor_loss}
\begin{aligned}
J_{\pi}(\phi)
= \mathbb{E}_{s \sim \mathcal{B}, a \sim \pi_\phi}
\Big[
  \mathcal{Z}_\theta(s,a) - \alpha \log \pi_\phi(a|s)
\Big].
\end{aligned}
\end{equation}

\end{definition}

\begin{figure}[t]
\centering
\begin{minipage}{0.95\linewidth}
\begin{algorithm}[H]
  \centering
    \caption{Collaborative Weighting Actor-Critic (CWAC)}
    \label{alg:CWAC_online}
    \begin{algorithmic}
      \STATE Initialize critic networks $Q_{\theta_1}$, $Q_{\theta_2}$ and actor-network $\pi_\phi$ with random parameters
      \STATE Initialize target networks $\theta^{'}_1 \leftarrow \theta_1$, $\theta^{'}_2 \leftarrow \theta_2$ and replay buffer $\B$
      \FOR{$t = 1$ to $T$}
        \STATE Calculate action $a \sim \pi_\phi(\cdot|s) $
        \STATE Get reward $r$ and new state $s^{'}$ and store transition tuple $(s, a, r, s^{'})$ in $\B$
        \STATE Sample mini-batch of transitions $(s, a, r, s^{'}) \sim \B$
        \STATE Get target value: $y= r + \gamma \left[ \min_{i=1,2}Z_{\theta^{'}_i}(s', a') - \alpha \log \pi_{\phi}(a' \mid s')\right], a' \sim \pi_{\phi}(\cdot \mid s')$ 
        \STATE Update critic with gradient descent by minimizing Eq.~\eqref{cwac:eq}
         \STATE Stochastic pessimistic value estimation via Eq.~\eqref{eq:pessimistic_q}
        \STATE Update actor with gradient ascent by maximizing Eq.~\eqref{eq:actor_loss}
        \STATE Update weights: $\theta^{'}_i \leftarrow  \tau \theta_i + (1 - \tau) \theta^{'}_i, i\in  1,2 $
      \ENDFOR
\end{algorithmic}
\end{algorithm}
\end{minipage}
\end{figure}

We present the detailed learning procedure of the proposed CWAC in Algorithm~\ref{alg:CWAC_online}. In contrast to~\citep{fujimoto2020equivalence}, our method does not rely on prioritized sampling over historical interactions. Instead of modifying the sampling distribution, it incorporates a pessimistic value sampling strategy together with a collaborative weighting mechanism that jointly and adaptively calibrates the influence of TD-errors and predictive uncertainty during critic optimization. Compared to the vanilla SAC algorithm, these modifications provide a systematic approach to mitigating overestimation bias in off-policy RL.

\section{Experiments}
\label{sec:experiment}
Our experiments are designed to address the following questions:
(1) Does CWAC achieve consistent and stable performance improvements across diverse environments? 
(2) How does the collaborative weighting mechanism facilitate Bellman error minimization and stabilize value learning? 
(3) Can CWAC stabilize value estimation under non-stationary reward perturbations?

To answer (1), we evaluate CWAC on 6 continuous control tasks from OpenAI Gym~\citep{brockman2016openai}, 4 tasks from PyBullet~\citep{coumans2016pybullet} and 12 tasks from DeepMind Control (DMC)~\citep{tassa2018deepmind} using a fixed set of hyperparameters.
For (3), we evaluate performance and TD-error dynamics in environments with stochastic reward noise.

\subsection{Experimental Setup}
\textbf{Hyperparameters.}
In Section~\ref{sec:method}, we introduce three key hyperparameters: the pessimism coefficient $\mu = 0.8$, the uncertainty weighting exponent $\beta_\omega = 2$, and the TD-error reweighting exponent $\beta_\xi = 1$. 
Unless otherwise specified, these hyperparameters are fixed across all experiments.
Our model architecture follows the standard design of SAC.
Further details can be found in Table~\ref{all_hyper}.

\textbf{Baselines.}
We compare CWAC against several representative actor-critic baselines, including value-improved actor-critic method VIAC~\citep{oren2025value}, the decision augmentation approach ALH~\citep{quangaugmenting}, the prioritized sampling method LAP~\citep{fujimoto2020equivalence}, the maximum-entropy actor-critic algorithm SAC~\citep{haarnoja2018soft}, and the deterministic policy gradient method TD3~\citep{fujimoto2018addressing}.

\subsection{Evaluation Environments}
We evaluate the effectiveness of the proposed algorithm on 6 Gym environments, 4 PyBullet environments, and 12 DMC benchmarks. Detailed descriptions of these experimental environments are provided below.

\begin{table}[!htbp] 
  \setlength{\tabcolsep}{1.0mm}
 \caption{State and action space dimensions of the continuous control environments used in experiments: Gym tasks (a), PyBullet tasks (b), and DMC tasks (c).}
  
  \begin{minipage}[b]{1\textwidth}
  \centering
  \resizebox{0.9\textwidth}{!}{
  \begin{tabular}{lllllll}
  \toprule
 Environment & Task & State Dim. & Action Dim. & Task & State Dim. & Action Dim.    \\
  \midrule
  \multirow{3}{*}{Gym}  &HalfCheetah &17&6& Ant  &111&8   \\
  &Hopper &11&3& Walker2d  &17&6   \\
  &Humanoid &376&17& BipedalWalker  &24&4   \\
  \bottomrule  
\end{tabular}
}
  \caption*{(a) }
  \end{minipage}
  \begin{minipage}[b]{1.0\textwidth}
  \centering
  \resizebox{1.\textwidth}{!}{
  \begin{tabular}{lllllll}
  \toprule
  Environment & Task & State Dim. & Action Dim. & Task & State Dim. & Action Dim.     \\
  \midrule
  \multirow{2}{*}{PyBullet}  &HalfCheetahBulletEnv &26&6& AntBulletEnv  &28&8   \\
  &HopperBulletEnv &15&3& Walker2DBulletEnv  &22&6   \\
     \bottomrule
  \end{tabular}
  }
  \caption*{(b)}
  \end{minipage}
  \begin{minipage}[c]{1\textwidth}
    \centering
    \resizebox{1.\textwidth}{!}{
    \begin{tabular}{lllllll}
    \toprule
    Environment & Task & State Dim. & Action Dim. & Task & State Dim. & Action Dim.    \\
    \midrule
    \multirow{6}{*}{DMC}  &reacher-hard &6&2& walker-walk  &24&6   \\
    &hopper-hop &15&4& hopper-stand  &15&4   \\
    &fish-swim &24&5& swimmer-swimmer6  &25&5   \\
    &pendulum-swingup &3&1& cheetah-run  &17&6   \\
    &walker-run &24&6& quadruped-walk  &78&12   \\
    &quadruped-run &78&12& finger-turn\_hard &12&2   \\
       \bottomrule
    \end{tabular}
    }
    \caption*{(c)}
    \end{minipage}
    
  \label{state_action_dims}
  \end{table}

\textbf{Gym.}
Gym is a fast and powerful physics engine that has become a standard platform for simulating complex robotic dynamics and control tasks. It is extensively utilized in the robotics and RL communities for training agents to perform tasks such as locomotion and manipulation, and it offers a diverse set of benchmarks for evaluating RL algorithms. 
    In our experiments, we adopt 6 challenging tasks from the Gym environments: HalfCheetah, Ant, Hopper, Walker2d, Humanoid, and BipedalWalker.

\textbf{PyBullet.}     
PyBullet is an open-source physics engine that supports general-purpose physical simulation as well as high-fidelity modeling for robotics tasks. It provides an efficient control interface for fast simulation of rigid-body dynamics, collision detection, and various control problems. Compared to MuJoCo, PyBullet typically exhibits higher levels of environment-induced dynamical noise, making it a more challenging testbed for evaluating the robustness of reinforcement learning algorithms. In this work, we evaluate the proposed algorithm on four PyBullet benchmark tasks: HalfCheetahBulletEnv, AntBulletEnv, HopperBulletEnv, and Walker2DBulletEnv.

\textbf{DMC.}
The DeepMind Control (DMC) suite is a set of continuous control environments mainly designed for robotic manipulation tasks, but it also covers a wide range of industrial control challenges. In our experiments, we evaluated the proposed algorithm on four continuous control tasks from the DMC suite: reacher-hard, walker-walk, hopper-hop, hopper-stand, fish-swim, swimmer-swimmer6, pendulum-swingup, cheetah-run, walker-run, quadruped-walk, quadruped-run, and finger-turnhard task.

\subsection{Performance Evaluation on Gym Environments}
We evaluate CWAC against several representative baselines on 6 continuous control tasks from the Gym benchmark suite. Learning curves are shown in Figure~\ref{fig:Over_results}, with final performance summarized in Table~\ref{tab:mojoco_res}.

Overall, CWAC demonstrates consistent and superior performance over VIAC, ALH, and LAP across all 6 Gym tasks. On the Ant task, VIAC exhibits a clear performance degradation after 1.5M training steps, which we attribute to its aggressive, greedy value improvement scheme that tends to amplify noise-dominated updates and induce non-stationary training dynamics. The ALH method constructs actions by combining representations from the previous and current timesteps, which may lead to temporally correlated actions and thus limit performance gains. 
LAP yields only marginal improvements over SAC and even shows performance degradation on Ant and BipedalWalker, suggesting that TD-error-based prioritized sampling may overemphasize noisy transitions, potentially resulting in suboptimal performance. In contrast, CWAC systematically mitigates overestimation bias through uncertainty-aware weighting and pessimistic value estimation, resulting in more stable and efficient performance improvements.

As shown in Table~\ref{tab:mojoco_res}, CWAC achieves the best overall performance among all benchmark methods, including VIAC, ALH, LAP, SAC, and TD3, with relative improvements of 15.3\%, 31.5\%, 29.5\%, 40.3\%, and 32.9\%, respectively, in average evaluation metrics.
Moreover, CWAC demonstrates consistent performance gains across all 6 Gym tasks, indicating that the proposed collaborative weighting mechanism enables more stable Q-value estimation and provides more reliable guidance for policy improvement.
\begin{figure*}[ht]
\begin{center}
\centerline{\includegraphics[width=\textwidth]{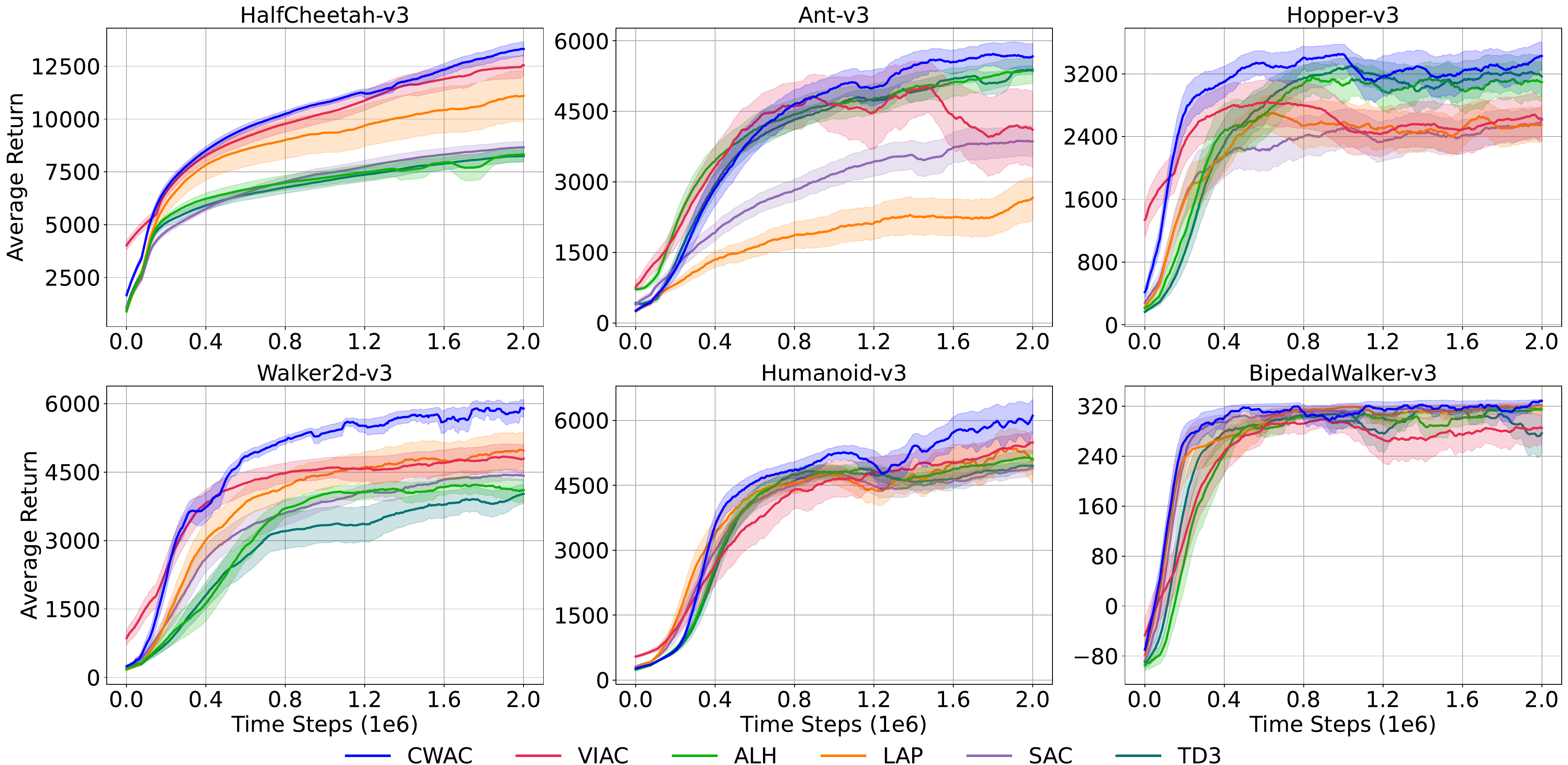}}
\caption{Learning curves on 6 Gym tasks. The shaded region represents half a standard deviation of the average evaluation over 10 trials. Curves are smoothed uniformly for visual clarity.}
\label{fig:Over_results}
\end{center}
\vskip -0.2in
\end{figure*}

\begin{table*}[!ht]
\centering
\setlength\tabcolsep{2.5pt}
\caption{Average return of the last 10 evaluation scores over 10 random seeds. The maximum values of each row are bolded. $\pm$ corresponds to a standard deviation over trials. }
  \resizebox{1.\textwidth}{!}{
\begin{tabular}{lcccccc}
\toprule
Tasks        &CWAC&VIAC&ALH&LAP&SAC&TD3     \\  \midrule
HalfCheetah-v3&\bf 13308$\pm$503  &12571$\pm$735&8339$\pm$484&11085$\pm$1813&8672$\pm$399& 8259$\pm$430  \\
Ant-v3&\bf 5764$\pm$356&4186$\pm$1208&5361$\pm$427&2687$\pm$672&3906$\pm$508& 5397$\pm$134  \\
Hopper-v3&\bf 3404$\pm$286&2517$\pm$466 &3130$\pm$367&2561$\pm$334 &2552$\pm$320&3216$\pm$372 \\
Walker2d-v3&\bf 5768$\pm$349 &4936$\pm$338&4159$\pm$303&4958$\pm$559&4419$\pm$172 &4008$\pm$322 \\
Humanoid-v3&\bf 6184$\pm$566 &5635$\pm$310&5117$\pm$280&5224$\pm$696&4892$\pm$230&5003$\pm$169  \\
BipedalWalker-v3&\bf 328$\pm$4  &290$\pm$34&315$\pm$4&321$\pm$10&322$\pm$4&265$\pm$57 \\
 \midrule
Avg.& 5792 &5022&4403&4472&4127&4358 \\
\bottomrule
\end{tabular}
 }
\label{tab:mojoco_res}
\end{table*}

\subsection{Performance on PyBullet Environments}

We evaluate CWAC and all baseline algorithms on the PyBullet benchmark suite. Each algorithm is independently trained with 10 random seeds for 2 million environment interaction steps. The learning curves are shown in Figure~\ref{fig:Over_results_pybullet_cwac}, and the final performance results are reported in Table~\ref{tab:mojoco_res_pybullet_cwac}.
\begin{figure*}[ht]
  \begin{center}
  \centerline{\includegraphics[width=1.0\textwidth]{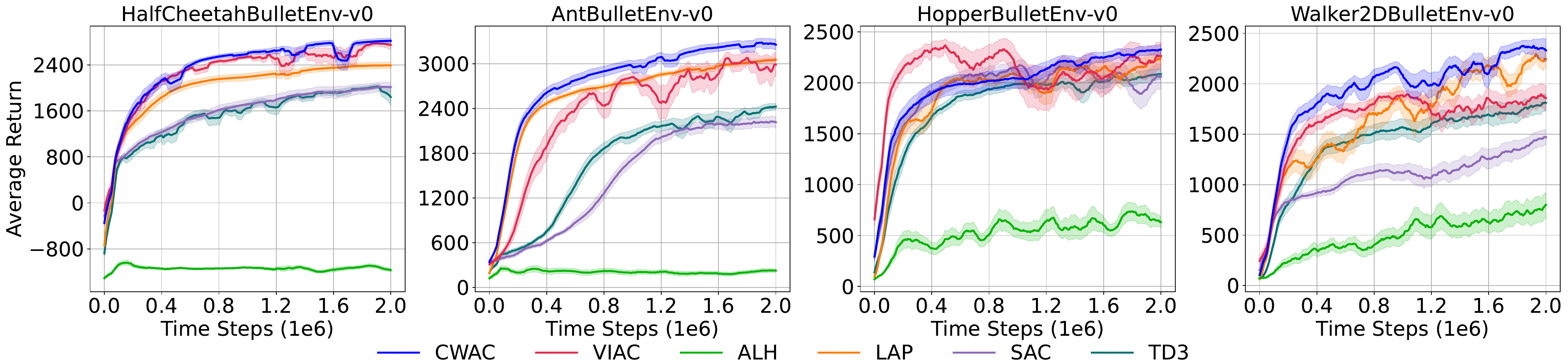}}
\caption{Learning curves on four PyBullet tasks. The shaded areas represent half of the standard deviation of the average evaluation returns over 10 independent trials. For better visualization, all curves are uniformly smoothed.}
  \label{fig:Over_results_pybullet_cwac}
  \end{center}
  \vskip -0.2in
  \end{figure*}

As shown in Figure~\ref{fig:Over_results_pybullet_cwac}, CWAC achieves faster policy convergence and superior performance compared with all baseline methods across the four tasks. Specifically, LAP obtains only marginal improvements on HalfCheetahBulletEnv, AntBulletEnv, and Walker2DBulletEnv, indicating that prioritized experience replay improves sample utilization efficiency but provides limited benefits for policy improvement. VIAC achieves higher returns on several tasks by adopting a more aggressive value estimation and value-guided optimization mechanism. However, such aggressive value updates may exacerbate value overestimation bias, thereby compromising training stability.
In contrast, CWAC consistently demonstrates stable and significant performance gains across all four tasks. These results suggest that the uncertainty-aware value weighting mechanism effectively mitigates the misleading effects of overestimated states during policy optimization. By leveraging valuable value information while reducing the adverse impact of value error accumulation and propagation, CWAC enables a more robust policy learning process.
Furthermore, as an extension of TD3, ALH underperforms the standard TD3 across all four PyBullet tasks. We hypothesize that this degradation may be attributed to introducing the implicit representation of the previous observation as an additional input to the actor network. Such a design encourages the policy to generate more similar actions for adjacent states, thereby reducing behavioral diversity and exploration capability during policy optimization, ultimately limiting further performance improvement.

\begin{table*}[!ht]
  \centering
  \setlength\tabcolsep{2.5pt}
      \caption{Average evaluation returns over the final 10 evaluation episodes on PyBullet environments across 10 random seeds. The maximum values of each row are bolded. $\pm$ corresponds to a standard deviation over trials.}
   \resizebox{1.\textwidth}{!}{
  \begin{tabular}{lcccccc} 
  \toprule

 Tasks      &CWAC&VIAC&ALH&LAP&SAC&TD3  \\  \midrule
  HalfCheetahBulletEnv-v0&\bf 2818$\pm$46  &2741$\pm$77& -1174$\pm$31&2398$\pm$56&2015$\pm$82& 1820$\pm$160\\
  AntBulletEnv-v0&\bf 3249$\pm$74 &3017$\pm$149&224$\pm$25&3051$\pm$32&2216$\pm$76& 2423$\pm$46 \\
  HopperBulletEnv-v0&\bf 2326$\pm$51& 2255$\pm$136 & 640$\pm$58&2243$\pm$79& 2075$\pm$124&2082$\pm$52 \\
  Walker2DBulletEnv-v0&\bf 2325$\pm$119& 1845$\pm$132 & 800$\pm$128&2234$\pm$88&1480$\pm$76& 1819$\pm$82\\
  \midrule
  Avg.& 2680 &2465&123&2482&1947&2036 \\
 \bottomrule
  \end{tabular}
}
  \label{tab:mojoco_res_pybullet_cwac}
  \end{table*}
From Table~\ref{tab:mojoco_res_pybullet_cwac}, we observe that SAC-based CWAC, VIAC, and LAP consistently outperform the original SAC on all four PyBullet continuous control tasks. Specifically, CWAC, VIAC, and LAP improve the average return over SAC by 37.6\%, 26.6\%, and 27.5\%, respectively. These results demonstrate that incorporating value information into policy optimization can effectively improve both learning efficiency and final performance.

 \subsection{Performance Evaluation on DMC Environments}
 \label{appendx:dmc}
We conduct experiments on the DMC benchmark suite, where each algorithm is trained for 500K environment steps with 10 random seeds. The aggregated learning curves are presented in Figure~\ref{fig:results_dmc}, and the final performance comparisons are summarized in Table~\ref{dmc_results}. 

Across the 12 evaluated tasks, CWAC achieves the SOTA performance on 8 tasks, highlighting its robustness and effectiveness in diverse continuous control scenarios. In particular, CWAC built upon SAC consistently demonstrates superior and more stable performance compared with the vanilla SAC. These improvements suggest that the proposed collaborative weighting and pessimistic value estimation mechanisms effectively alleviate value estimation errors, leading to improved sample efficiency and more reliable policy optimization.

\begin{figure*}[ht]
\begin{center}
\centerline{\includegraphics[width=\textwidth]{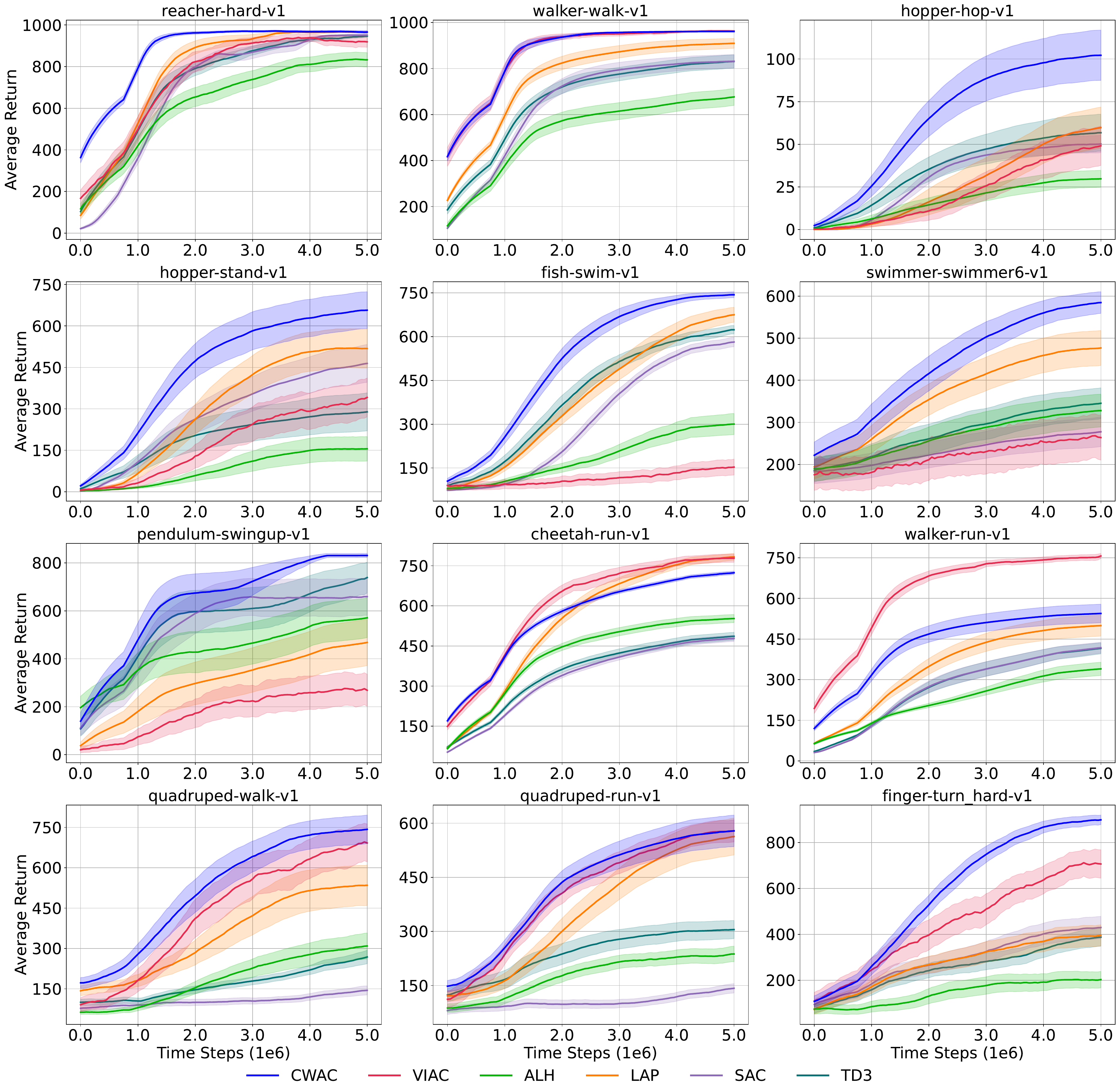}}
\caption{Learning curves on 12 DMC tasks. The shaded region represents half a standard deviation of the average evaluation over 10 trials. Curves are smoothed uniformly for visual clarity.}
\label{fig:results_dmc}
\end{center}
\vskip -0.2in
\end{figure*}

  \begin{table*}[!ht]
    \centering
    \setlength\tabcolsep{3.0pt}
          \caption{Average evaluation returns over the final 10 evaluation episodes on DMC environments across 10 random seeds. The maximum values of each row are bolded. $\pm$ corresponds to a standard deviation over trials.}
           \resizebox{1.\textwidth}{!}{
    \begin{tabular}{lcccccc} 
    \toprule
  
   Tasks      &CWAC&VIAC&ALH&LAP&SAC&TD3   \\  \midrule
   reacher-hard-v1 & \bf 967$\pm$7 & 904$\pm$33 & 824$\pm$36 & 966$\pm$4 & 953$\pm$3 & 947$\pm$5 \\
walker-walk-v1 & 961$\pm$2  & \bf 968$\pm$2 & 685$\pm$38 & 914$\pm$21 & 835$\pm$28 & 834$\pm$30 \\
hopper-hop-v1 & \bf 103$\pm$15  & 55$\pm$12 & 30$\pm$5 & 62$\pm$12 & 50$\pm$5 & 58$\pm$11 \\
hopper-stand-v1 & \bf 660$\pm$67 & 357$\pm$69 & 155$\pm$43 & 526$\pm$69 & 471$\pm$69 & 295$\pm$70 \\
fish-swim & \bf 748$\pm$9  & 162$\pm$31 & 310$\pm$38 & 689$\pm$23 & 593$\pm$13 & 628$\pm$15 \\
swimmer-swimmer6 & \bf 590$\pm$25  & 268$\pm$53 & 333$\pm$39 & 478$\pm$42 & 284$\pm$35 & 348$\pm$37 \\
pendulum-swingup & \bf 832$\pm$8 & 276$\pm$67 & 580$\pm$82 & 483$\pm$ 96 & 659$\pm$78 & 752$\pm$63 \\
cheetah-run & 728$\pm$7 & 781$\pm$16 & 556$\pm$15 & \bf  788$\pm$13 & 482$\pm$11 & 490$\pm$14 \\
walker-run & 547$\pm$35 & \bf  761$\pm$9 & 345$\pm$25 & 504$\pm$39 & 422$\pm$19 & 422$\pm$19 \\
quadruped-walk & \bf  744$\pm$53 & 729$\pm$69 & 315$\pm$50 & 536$\pm$75 & 156$\pm$18 & 281$\pm$28 \\
quadruped-run & 584$\pm$44 & \bf  590$\pm$28 & 243$\pm$22 & 571$\pm$51 & 146$\pm$12 & 308$\pm$25 \\
finger-turn\_hard & \bf  909$\pm$17 & 719$\pm$67 & 216$\pm$37 & 399$\pm$46 & 432$\pm$48 & 395$\pm$41 \\

    \midrule
     
    Avg.& 698 &548&383&576&456& 479 \\
   \bottomrule
    \end{tabular}
     }
     \label{dmc_results}
\end{table*}

\subsection{Integration with More Off-Policy Algorithms}
To further validate the generality of the proposed CWAC algorithm, we integrate CWAC with two representative off-policy value-based algorithms, namely TD3 and DDPG, and conduct corresponding experimental studies. The experiments are carried out on six Gym continuous control environments, with each environment independently run for 10 trials. All algorithms adopt a unified hyperparameter configuration and identical experimental conditions to ensure a fair comparison. Figure~\ref{fig:results_gym_value_based} presents the learning curve comparisons among CWAC(TD3), CWAC(DDPG), and the original TD3 and DDPG algorithms.
\begin{figure}[!htbp]
  \begin{center}
  \centerline{\includegraphics[width=1.0\textwidth]{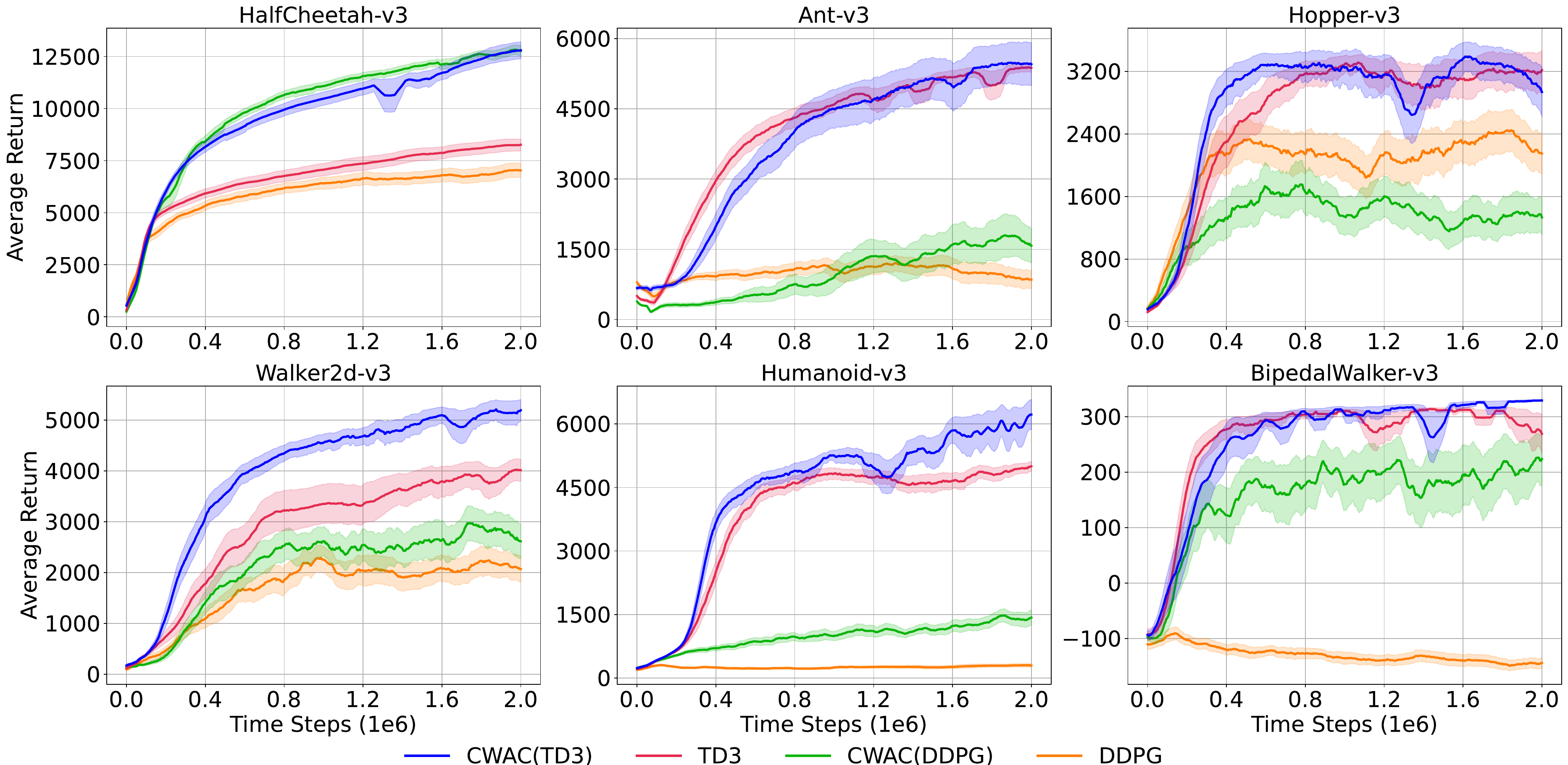}}
\caption{Learning curves of CWAC, TD3, and DDPG on six Gym continuous control tasks. Shaded regions denote half a standard deviation of the average evaluation over 10 trials. All curves are uniformly smoothed for visual clarity.}
  \label{fig:results_gym_value_based}
  \end{center}
  
  \end{figure}

As shown in Figure~\ref{fig:results_gym_value_based}, integrating CWAC with TD3 consistently improves learning performance across most evaluated tasks. Compared with the original TD3, CWAC(TD3) achieves significant performance gains on three out of six benchmark environments, while maintaining comparable performance on the remaining tasks without introducing performance degradation. Specifically, on the HalfCheetah, Walker2d, and Humanoid environments, CWAC(TD3) demonstrates faster convergence and achieves higher final cumulative returns. Moreover, CWAC(DDPG) also yields substantial improvements on these three tasks, further validating the effectiveness and generalizability of the proposed method across different value-based off-policy algorithms.
Interestingly, we observe that vanilla DDPG fails to achieve stable convergence on the BipedalWalker task, whereas CWAC(DDPG) exhibits consistent performance improvements throughout training. This result suggests that the proposed method can effectively alleviate unstable optimization issues in value-based policy learning.
Overall, these empirical results demonstrate that the collaborative weighting mechanism introduced in CWAC effectively improves the utilization efficiency of high-value experience samples in off-policy reinforcement learning.
Through the adaptive coordination of TD-errors and predictive uncertainty during optimization, CWAC effectively balances the influence of different training samples, leading to more stable policy learning, higher sample efficiency, and improved final policy performance.
These findings indicate that CWAC is not limited to maximum-entropy reinforcement learning frameworks, but can also serve as a general enhancement module for a broader class of policy gradient algorithms.

\subsection{Sensitivity Analysis of Core Hyperparameters}
 \label{appendx:abla}
CWAC introduces three hyperparameters to control the strength of stochastic sampling $\mu$, the uncertainty weighting coefficient $\beta_\omega$, and the temporal-difference error weighting coefficient $\beta_\xi$, respectively. To evaluate the sensitivity of these hyperparameters, we conduct experiments on the Walker2d task, where each configuration is trained for 2 million time steps using 10 different random seeds, and analyze the resulting training dynamics under varying parameter settings.

\begin{figure}[!htbp]
\begin{center}
\centerline{\includegraphics[width=1.0\textwidth]{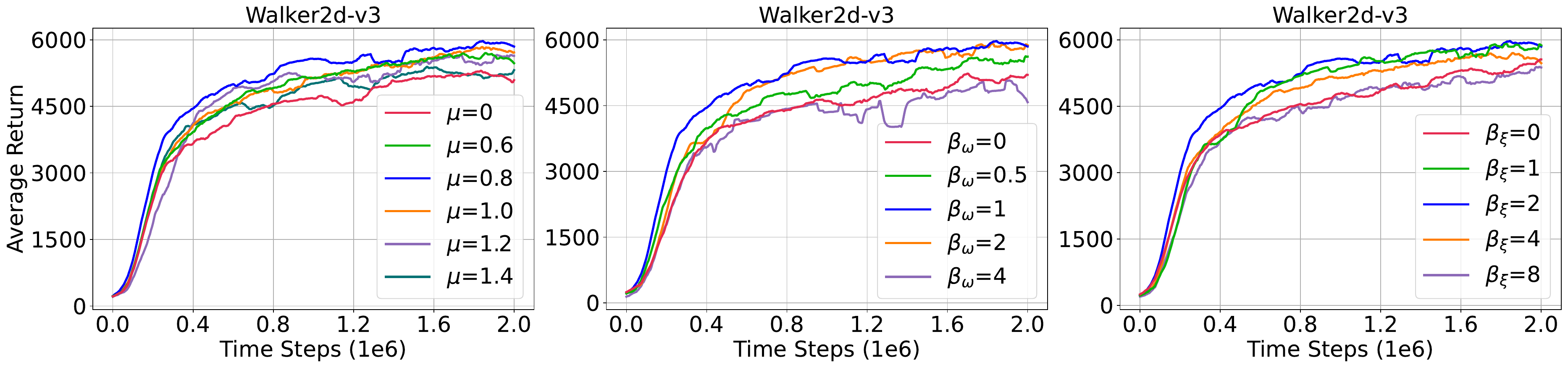}}
\caption{Hyperparameter sensitivity analysis. The shaded region represents half a standard deviation of the average evaluation over 10 trials. Curves are smoothed uniformly for visual clarity.}
\label{fig:results_ablation_mu_beta}
\end{center}

\end{figure}
\textbf{Pessimistic value coefficient $\mu$.}
The parameter $\mu$ controls the strength of stochastic pessimism, as described in Section~\ref{sec:method}. In this study, we set $\beta_\omega=1$ and $\beta_\xi=2$ as the default setting. When $\mu=0$, the stochastic pessimistic sampling mechanism is disabled. 
When $\mu$ lies within the interval $[0.6, 1.2]$, CWAC demonstrates consistently stable performance improvements across different choices of $\mu$, as shown in Figure~\ref{fig:results_ablation_mu_beta} (left). However, overly large values of $\mu$ may amplify estimation bias and introduce training instability, as observed for CWAC with $\mu=1.4$.

\textbf{Uncertainty coefficient $\beta_\omega$.}
The coefficient $\beta_\omega$ controls the weighting of TD-error according to the uncertainty in the Q-value estimates, directly modulating each sample’s contribution. When $\beta_\omega=0$, no uncertainty weighting is applied, resulting in a notable performance drop compared to settings with $\beta_\omega \in [0.5, 2]$, as shown in Figure~\ref{fig:results_ablation_mu_beta} (center). As $\beta_\omega$ increases, TD-errors associated with high uncertainty are increasingly down-weighted, while those with low uncertainty are correspondingly up-weighted, as described in Section~\ref{sec:method}. However, further increasing $\beta_\omega$ (e.g., to 4) amplifies the overall fluctuations of TD-errors, which in turn leads to unstable performance.

\textbf{TD-error coefficient $\beta_\xi$.}
Similarly, as shown in Figure~\ref{fig:results_ablation_mu_beta} (right), the weighting of samples with high TD-error is suppressed, while that of samples with low TD-error is increased. As $\beta_\xi$ increases, the variability of sample-wise error uncertainty is amplified, which can adversely affect training stability and lead to degraded performance.

\subsection{Value Estimation Error Visualization}
To further investigate the training stability of the proposed method, we visualize the Q-value estimation error $Q(s, a)-G^\pi(s, a)$ over 2M training steps across 6 continuous control tasks. 
The empirical results are illustrated in Figure~\ref{fig:results_q_value}. 

\begin{figure*}[ht]
\begin{center}
\centerline{\includegraphics[width=\textwidth]{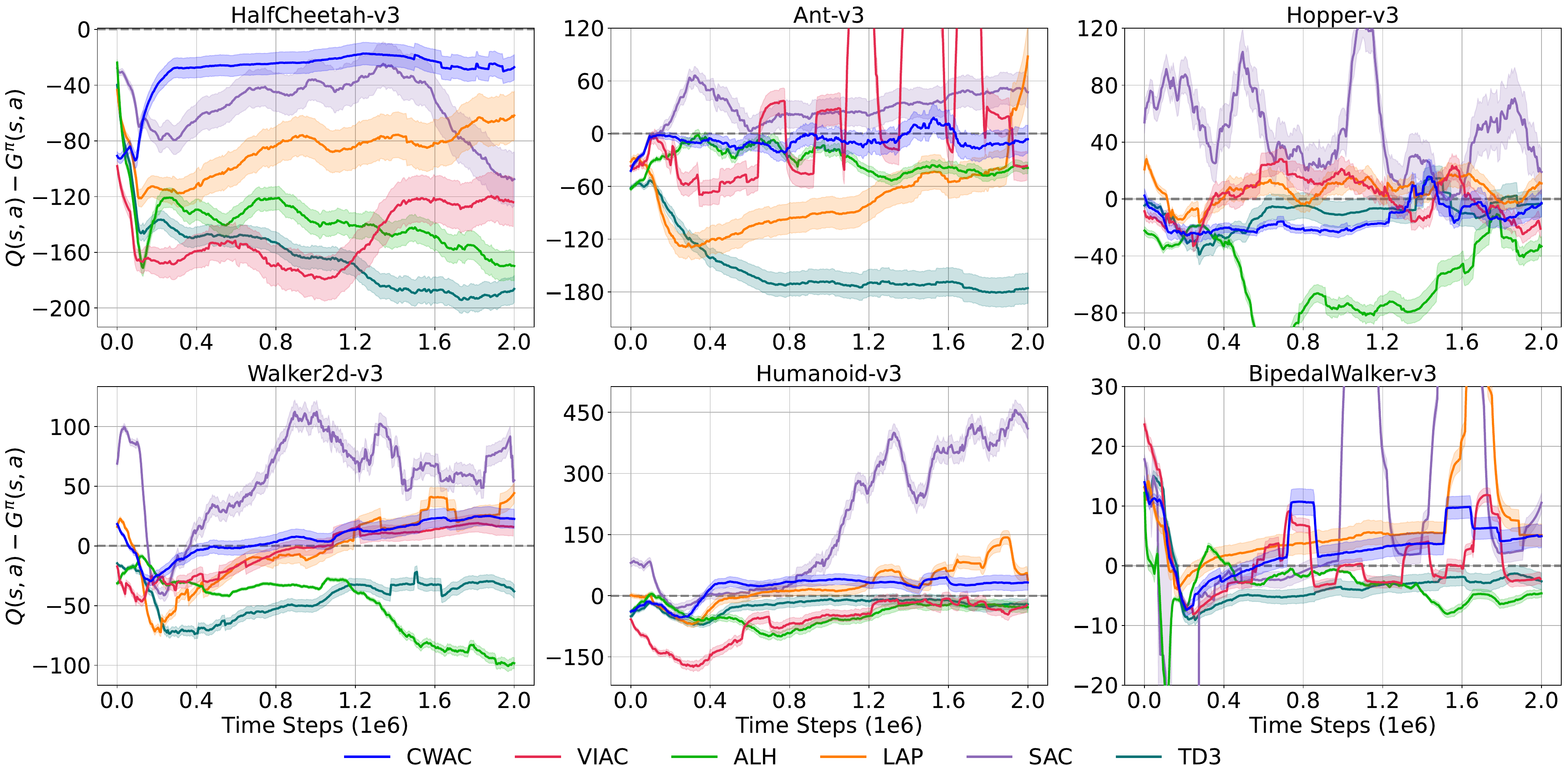}}
\caption{Value estimation error on 6 Gym tasks.
The shaded region represents half a standard deviation of the average evaluation over 10 trials. Curves are smoothed uniformly for visual clarity. The black dashed line indicates the training target of the value function.}
\label{fig:results_q_value}
\end{center}
\vskip -0.2in
\end{figure*}
 
Overall, CWAC exhibits consistently smoother and more stable TD-error convergence across all benchmarks, indicating improved stability in critic optimization. In contrast, conventional off-policy baselines such as SAC and TD3 show elevated TD-error during the early stages of training, accompanied by pronounced stochastic oscillations, which suggest unstable value estimation.
LAP reduces TD-error relative to SAC, particularly in high-dimensional tasks such as HalfCheetah, Walker2d, and Humanoid. However, its prioritized sampling strategy may overemphasize noisy transitions, potentially amplifying aleatoric noise and inducing distributional shift, which can introduce estimation bias and destabilize policy updates.
We observe that ALH exhibits pronounced value underestimation on the Hopper and BipedalWalker tasks compared to TD3. This phenomenon can be attributed to its design, as ALH introduces a decision augmentation mechanism with hypothesis modeling, building upon TD3’s value smoothing to further increase conservatism in value estimation. While this design improves the stability of policy updates, it reduces adaptability to new data, thereby inducing systematic underestimation in value predictions.
For VIAC, we observe pronounced TD-error fluctuations on the Ant task, which correlate with its performance degradation. This instability likely stems from its aggressive value improvement operator, potentially compromising the consistency of the learning dynamics.
In contrast, CWAC maintains low-variance and well-behaved TD-error trajectories throughout training, effectively balancing conservative estimation with learning efficiency.

\subsection{Value Learning in Noisy Reward Environments}
To investigate the robustness of CWAC under non-stationary reward perturbations, we compare its value learning dynamics against the vanilla SAC algorithm on the HalfCheetah, Ant, and Hopper tasks. Specifically, we introduce stochastic perturbations to the reward signal during training by injecting Gaussian noise with a positive mean ($\epsilon_r \sim \mathcal{N}(1, 0.2)$) into the immediate reward $r$. The noise is applied with a frequency of 2\%, i.e., once every 50 environment interactions. The resulting transitions are stored in the replay buffer $\mathcal{B}$.
Validation performance and estimation error are compared under noisy and noise-free rewards, with the complete learning dynamics depicted in Figure~\ref{fig:results_gym_noise_value}.

\begin{figure}[!htbp]
\begin{center}
\centerline{\includegraphics[width=1.0\textwidth]{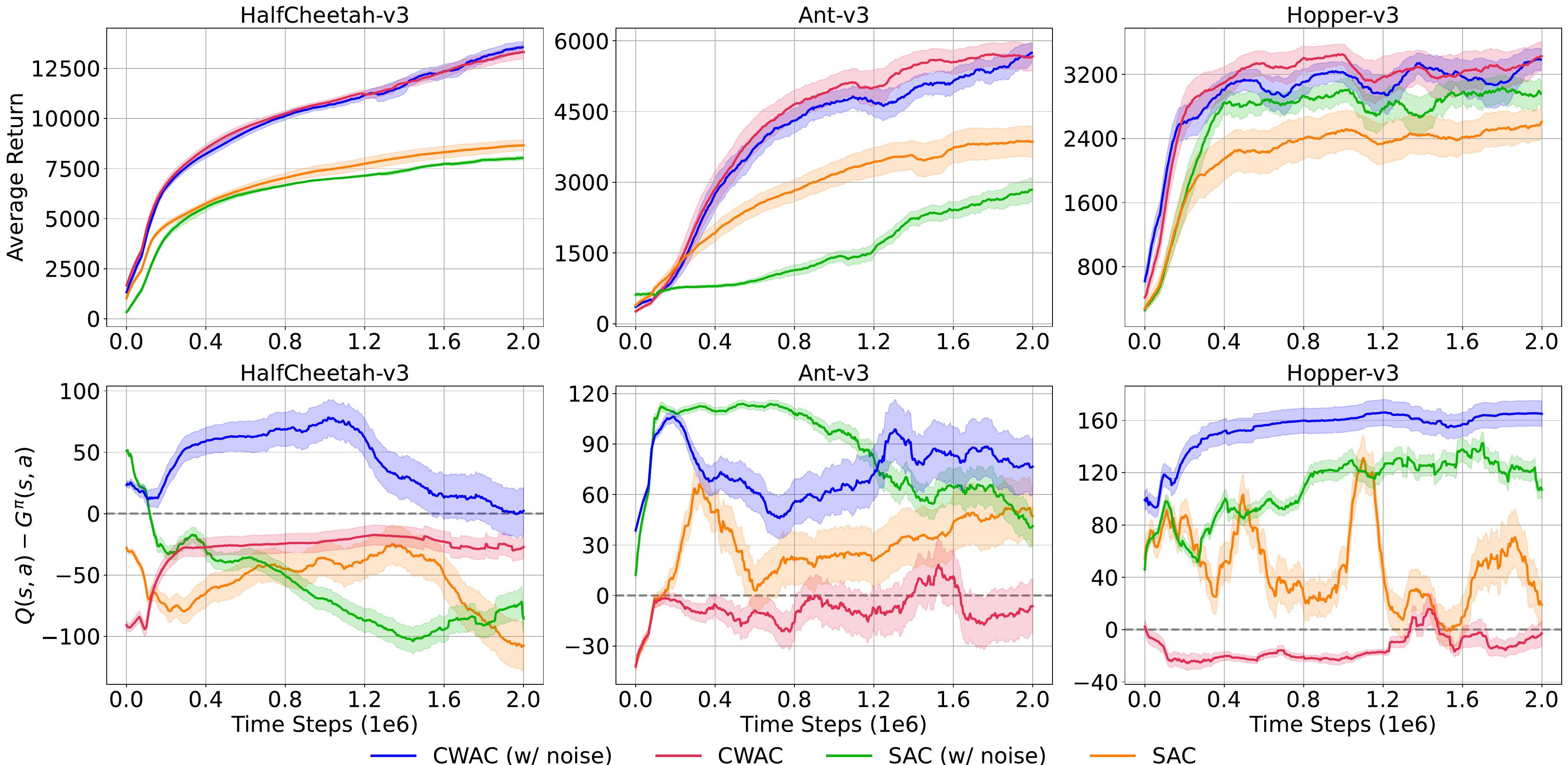}}
\caption{Comparison of evaluation performance and estimation error in noisy-reward and noise-free environments. The shaded region represents half a standard deviation of the average evaluation over 10 trials. Curves are smoothed uniformly for visual clarity. The black dashed line indicates the training target of the value function.}
\label{fig:results_gym_noise_value}
\end{center}

\end{figure}
From the results, we observe that CWAC (w/ noise) exhibits a significant increase in TD-error across all three tasks, while its performance remains comparable to the noise-free CWAC variant. In contrast, SAC (w/ noise) shows degraded performance on the HalfCheetah and Ant tasks compared to its noise-free counterpart, suggesting that non-uniform reward noise induces non-stationarity in the value signal, which leads to suboptimal policy performance.
Interestingly, on the Hopper task, SAC (w/ noise) outperforms its noise-free baseline. This suggests that overestimation in value estimation may encourage exploration, leading to improved performance.
Furthermore, we observe that CWAC (w/ noise) exhibits smoother TD-error trajectories than SAC (w/ noise), demonstrating robustness to noisy reward signals. In contrast, SAC (w/ noise) shows pronounced estimation error oscillations during the early stages, particularly on the HalfCheetah and Ant tasks, indicating that the injected noise introduces non-stationarity into value estimates, which slows policy convergence.

\begin{table*}[!htbp]
\setlength\tabcolsep{2.5pt}
\centering
\caption{A complete comparison of hyper-parameter choices between CWAC and 5 baselines, including VIAC, ALH, LAP, SAC, and TD3.}
\begin{tabular}{lcccccc}
\toprule
\bf{Hyper-parameter} & CWAC&VIAC&ALH& LAP&SAC  &TD3\\
\midrule
Optimizer  & Adam& Adam& Adam&Adam& Adam& Adam \\
Critic learning rate & $3 \cdot 10^{-4}$& $3 \cdot 10^{-4}$& $3 \cdot 10^{-4}$& $3 \cdot 10^{-4}$& $3 \cdot 10^{-4}$& $3 \cdot 10^{-4}$ \\
Actor learning rate  & $3 \cdot 10^{-4}$& $3 \cdot 10^{-4}$ & $3 \cdot 10^{-4}$&$3 \cdot 10^{-4}$&$3 \cdot 10^{-4}$&$3 \cdot 10^{-4}$ \\
Target update rate $\tau$  & $5 \cdot 10^{-3}$& $5 \cdot 10^{-3}$& $5 \cdot 10^{-3}$& $5 \cdot 10^{-3}$ & $5 \cdot 10^{-3}$& $5 \cdot 10^{-3}$ \\ 
Batch size  & $256$& $256$& $256$& $256$ & $256$& $256$\\ 
Discount factor & $0.99$ & $0.99$ & $0.99$ & $0.99$ & $0.99$& $0.99$ \\
Number of critics  & 2& 2 &2&2 & 2& 2\\
Hidden dimension  & $256$& $256$& $256$& $256$ & $256$& $256$ \\ 
Start timesteps  & 1e4& 1e4& 25e3& 25e3 & 25e3& 25e3 \\
Policy update frequency $d$& 1 &1& 2&1&1&2\\
 Policy noise& - &- & $\N(0, 0.2)$ &-& -  &$\N(0, 0.2)$ \\ 
 Noise clip range  &-& - &  $[-0.5,0.5]$&-& - & $[-0.5,0.5]$  \\
  Priority scaling factor $p$  &-& - &  -&0.4& - & - \\
  Value function  &Huber& expectile &  MSE&Huber& MSE & MSE \\
 Pessimistic value coefficient $\mu$  &0.8& - &  -&-& - & - \\
  Uncertainty coefficient $\beta_\omega$ &1& - &  -&-& - & - \\
  TD-error coefficient $\beta_\xi$&2& - &  -&-& - & - \\
\bottomrule
\end{tabular}
\label{all_hyper}
\end{table*}

\subsection{Runtime Analysis}
\label{sec:runtime_analysis}

To evaluate the computational efficiency of the proposed CWAC algorithm, we conducted comprehensive experiments on the HalfCheetah-v3 task. The experimental environment was equipped with a 56-core CPU and an Nvidia RTX 4090 GPU running on Linux. Each algorithm was trained for 2 million time steps to ensure a fair comparison of total wall-clock time.

\begin{table*}[!ht]
\centering
\caption{Comparison of total training duration (Wall-clock time) across baseline algorithms on the HalfCheetah-v3 task.}
\label{tab:consumer_time}
\begin{tabular}{lcccccc}
\toprule
\textbf{Methods} & CWAC & VIAC & ALH & LAP & SAC &TD3 \\
\midrule
Runtime (Hours) & 7.3 & 8.9 & 6.4 & 6.8 & 6.6 & 5.1 \\
\bottomrule
\end{tabular}
\end{table*}

As illustrated in Table~\ref{tab:consumer_time}, CWAC introduces only a marginal increase in training duration compared to the vanilla SAC (7\%). This modest overhead is primarily attributed to the gradient updates for the distributional Q-networks and the corresponding computation of uncertainty parameters used in collaborative weighting. These results demonstrate that CWAC attains superior stability and performance without incurring prohibitive computational overhead, making it practical for high-dimensional control tasks.

\section{Conclusions}
\label{conclusion}
\vspace{-1em}
We propose CWAC, a dynamic weighting framework that integrates two complementary mechanisms: stochastic pessimistic value estimation and a collaborative weighting scheme. By jointly weighting the TD-error and predictive uncertainty under a stochastic pessimistic value estimation formulation, CWAC produces more reliable value estimates and facilitates stable policy optimization. Extensive experiments demonstrate that, using a single set of hyperparameters, CWAC consistently suppresses value overestimation while maintaining stable value learning across a wide range of continuous control benchmarks. Comprehensive ablation studies further verify the effectiveness of each individual component, confirming that their synergy is essential to the observed performance improvements.

Despite these advantages, stochastic pessimistic value estimation inevitably introduces a degree of conservatism. Although such conservatism effectively mitigates overestimation bias, excessive pessimism may lead to suboptimal policy performance by discouraging beneficial exploration. An important direction for future work is therefore to develop unbiased or adaptively calibrated value learning mechanisms that better balance overestimation suppression with effective exploration, thereby achieving both robust value estimation and improved policy performance.

\section{Declarations}
\subsection{Funding}
This study was supported by National Major Science and Technology Special Project (CN) (Grant Number 2025ZD1604900).

\subsection{Conflicts of interest/Competing interests
}
The authors declare that they have no competing interests.
\subsection{Availability of data and material (data transparency)
}
The data that has been used is confidential.

\subsection{Code availability (software application or custom code)
}
The code used in this study is available upon request to the corresponding
author.

\subsection{Authors' contributions}
Gong Gao: Conceptualization, Methodology, Visualization, Writing – original draft,
Writing – review \& editing . Xiao Lai*: Project administration, Resources. Ziqi Xie: Writing – review \& editing. Guojie Chen:
Writing – original draft.  Xianhui Liu: Validation, Funding acquisition.
Weidong Zhao: Validation, Funding acquisition.

\bibliographystyle{elsarticle-num}

\bibliography{cas-refs}



 

 



\end{document}